\documentclass[final,12pt,cleveref]{colt2026_preprint} 
\PassOptionsToPackage{capitalize}{cleveref}

\usepackage[textsize=tiny
]{todonotes}
\title[Algorithmic Stability of Minimum-Norm Interpolating Deep Neural Networks]{Sufficient Conditions for Stability of Minimum-Norm Interpolating Deep ReLU Networks}
\usepackage{times}
\usepackage{cleveref}
\newtheorem{hypothesis}[theorem]{Hypothesis}
\newtheorem{assumption}[theorem]{Assumption}
\newtheorem{mainresult}[theorem]{Main result}

\coltauthor{%
 \Name{Ouns El Harzli} \Email{ouns.el.harzli@cs.ox.ac.uk}\\
 \addr University of Oxford
 \AND
 \Name{Yoonsoo Nam} \Email{yoonsoo.nam@physics.ox.ac.uk}\\
 \addr University of Oxford
 \AND
 \Name{Ilja Kuzborskij} \Email{iljak@google.com}\\
 \addr Google DeepMind
 \AND
 \Name{Bernardo Cuenca Grau} \Email{bernardo.cuenca.grau@cs.ox.ac.uk}\\
 \addr University of Oxford
 \AND
 \Name{Ard A. Louis} \Email{ard.louis@physics.ox.ac.uk}\\
 \addr University of Oxford
}

\newcommand{\bX}{\mathbf{X}}
\newcommand{\bx}{\mathbf{x}}
\newcommand{\by}{\mathbf{y}}
\newcommand{\bz}{\mathbf{z}}
\newcommand{\bu}{\mathbf{u}}
\newcommand{\bv}{\mathbf{v}}

\newcommand{\bW}{\mathbf{W}}
\newcommand{\bA}{\mathbf{A}}

\newcommand{\cT}{\mathcal{T}}

\newcommand{\cA}{\mathcal{A}}

\renewcommand{\th}{\theta}

\newcommand{\R}{\mathbb{R}}
\newcommand{\tp}{^\top}

\newcommand{\cX}{\mathcal{X}}

\newcommand{\ip}[1]{\langle #1 \rangle}
\newcommand{\pr}[1]{\left( #1 \right)}

\newcommand{\cbr}[1]{\left\{ #1 \right\}}
\newcommand{\abs}[1]{\left| #1 \right|}

\renewcommand{\P}{\mathbb{P}}

\newcommand*\diff{\mathop{}\!\mathrm{d}}
\newcommand{\Var}{\mathrm{Var}}

\DeclareMathOperator*{\argmin}{arg\,min}

\newcommand{\relu}{\mathrm{ReLU}}
\newcommand{\QED}{\hfill\ensuremath{\square}}
\newcommand{\cy}{1}
\newcommand{\cymult}{}
\newcommand{\net}{N}
\newcommand{\fsub}{f}

\begin{document}

\maketitle


\begin{abstract}
Algorithmic stability is a classical framework for analyzing the generalization error of learning algorithms.  
 It predicts that an algorithm has small generalization error if it is insensitive to small perturbations in the training set such as the removal or replacement of a training point.
  While stability has been demonstrated for numerous well-known algorithms, this framework has had limited success in analyses of deep neural networks.
  In this paper we study the algorithmic stability of deep ReLU homogeneous neural networks that achieve zero training error using parameters with the smallest $L_2$ norm, also known as the minimum-norm interpolation, a phenomenon that can be observed in overparameterized models trained by gradient-based algorithms. We investigate sufficient conditions for such networks to be stable.
  We find that 1) such networks are stable when they contain a (possibly small) stable sub-network, followed by a layer with a low-rank weight matrix, and 2) such networks are not guaranteed to be stable even when they contain a stable sub-network, if the following layer is not low-rank. 
  The low-rank assumption is inspired by recent empirical and theoretical results which demonstrate that training deep neural networks is biased towards low-rank weight matrices, for minimum-norm interpolation and weight-decay regularization.   
\end{abstract}


\section{Introduction}

%
%
\paragraph{Background and motivation} 

The generalization ability of gradient-based algorithms is reasonably well-understood when learning involves solving convex or quasi-convex optimization problems, however the picture is much less clear for non-convex learning  problems involving overparameterized models, such as that of training a deep neural network.

In this paper, we explore the generalization capabilities of deep neural networks from the viewpoint of \emph{algorithmic stability}. Stable learning algorithms are insensitive to small perturbations (e.g. removal or replacement of data points) of the training set and  generalize well under mild assumptions ~\citep{bousquet2002stability,shalev2010learnability}.
%
%

%
Many algorithms have been  shown to be stable, including nonparametric predictors (e.g.\, nearest-neighbors)~\citep{devroye1979distribution}, minimizers of strongly convex problems (such as the ridge regression estimator)~\citep{bousquet2002stability}, GD-type algorithms minimizing convex and smooth objectives~\citep{hardt2016train,lei2020fine}, as well as quasi-convex objectives~\citep{charles2018stability,richards2021stability}.

Despite these advances, success has so far been limited in the context
of deep neural networks.
Although several works have analyzed the stability of stochastic gradient descent (SGD) in the
non-convex setting, their results come with significant limitations.
First, vacuous bounds are typically obtained when the number of
training steps (or time) far exceeds the sample size, i.e.\ $t \gg n$
\citep{hardt2016train,kuzborskij2017data,richards2021weakly,wang2023generalization};
at the same time, for a large enough model capacity, we expect
interpolation to happen when $t \gg n$.
Second, strong assumptions such as Lipschitzness of the loss function
\emph{in the parameters} or penalization of the
objective~\citep{hardt2016train,kuzborskij2017data,farghly2021time}
are often required.  Third, these works assume \emph{parameter
  stability}, meaning that parameters are expected to remain close
(typically in the Euclidean distance) if the training set is perturbed
slightly; this is often unrealistic in the context of neural networks
since the same predictor can be expressed using very different
parameters thanks to symmetries in the weight matrices and because of
non-convexity of the objective function.
Finally, these works focus on optimization aspects, which seldom
reveal insight into the structural properties of neural networks
obtained by stable algorithms. 

In this paper, we hypothesize that stability of deep nonlinear networks originates in the early layers, and is preserved throughout the subsequent layers. Specifically, we study the algorithmic stability of minimum-norm interpolation with deep ReLU homogeneous neural networks, and investigate sufficient conditions for stability in relation to the low-rank bias. Our analysis is motivated by recent theoretical and empirical findings suggesting that deep interpolating neural networks exhibit a low-rank weight matrix structure~\citep{frei2022implicit,timor2023implicit,galanti2023characterizing}.
Thus, the idea behind the proofs is to ensure that if a stable sub-network exists, a low-rank weight matrix will preserve its stability (\Cref{stable sub network with stable rank}), and by contrast, that non low-rank matrices could break its stability (\Cref{unstable network high rank}). 

For the data-generating process, we assume the existence of a neural
network with the same architecture that can interpolate the data using
bounded weights, although not necessarily with the minimal norm. This
implies a setting where data cannot be arbitrarily complex
ensuring that weights of such a network remain bounded as the training
set grows.

In this work, we also depart from optimisation-based analyses, by assuming that minimum-norm interpolation is accessible through an oracle. Indeed, it has been shown theoretically that the gradient flow (GF) algorithm (gradient descent can be seen as a discretization of GF) asymptotically fits an interpolating ReLU neural network with the minimum $L_2$ norm of the parameters \citep{lyu2019gradient,ji2020directional,phuong2020inductive}.  With this idealised viewpoint, already taken in \citet{timor2023implicit}, we aim to uncover structures of learned networks that favor (or not) stability. 

\paragraph{Technical contributions}

Studying the stability of minimum-norm interpolating deep neural networks comes with several technical challenges. Since the objective is non-convex, minimum-norm solutions are difficult to characterize. For example, first-order optimality conditions (KKT points) do not distinguish between local minima, global minima, and saddle points. Consequently, even if we assume access to an oracle that returns a global minimum-norm interpolating solution, this alone does not provide meaningful geometric properties in parameter space beyond the fact that its norm is smaller than that of any other interpolant. This issue is all the more problematic for stability analysis because, by definition, assessing stability consists, in our setting, in calling the oracle a second time after changing one data point. Since we don't assume parameter stability, two minimum-norm solutions have \emph{a priori} no reason to share any structural property in parameter space.

%

%
In this work, we establish that minimum-norm interpolations of deep ReLU homogeneous neural networks are stable when they contain a contiguous, \emph{stable sub-network} (for instance, the few first layers), and  this sub-network is followed by a \emph{low-rank weight matrix} (\Cref{stable sub network with stable rank}). At the same time, existence of a stable sub-network is not sufficient to guarantee stability of the deep network: we find that if the stable sub-network is not followed by a low-rank layer, subsequent layers can break stability (\Cref{unstable network high rank}), suggesting a tight interplay between stability and low-rank bias. 

The relationship between the stability of a sub-network and the stability of the full network is untrivial precisely because of the reasons exposed above: parameters $\hat \th$ of a minimum-norm solution and parameters $\hat \th^{(i)}$ of a second minimum-norm solution (after resampling one data point) could be, in principle, vastly different. Even if we assume that the sub-network $f_k (\hat \th, \cdot)$ up to layer $k$ is stable, the deeper layers $\hat \bW_{k+1}, ..., \hat \bW_{L}$ and $\hat \bW_{k+1}^{(i)}, ..., \hat \bW_{L}^{(i)}$ differ and there is no \emph{a priori} reason to believe that they would preserve any stability that originates in early layers. This issue is resolved with a low-rank layer thanks to several key arguments: 
\begin{itemize}
    \item a low-rank layer $\hat \bW_{k+1} = \lambda_k \bu_k \bv_k^T + \bW_k^\epsilon$ projects learned representations onto the dominating singular direction $\bv_k$. Specifically, the sign of the product between $\bv_k^T f_k (\hat \th, \cdot)$ and $\net (\hat \th, \cdot)$ is independent of the input, modulo the perturbation $\bW_k^\epsilon$;
    \item the pertubation $\bW_k^\epsilon$ associated with directions with small singular values can be controlled as it propagates through deeper layers; and  
    \item minimum-norm interpolation forces a margin at the sub-network level $|\bv_k^T f_k (\hat \th, \bx)| \geq \alpha$. Intuitively, the learned representation cannot be too close to the decision boundary because it would otherwise require large coefficients, and hence a large norm, in subsequent layers to satisfy the interpolation.
\end{itemize}

These arguments combined yield a tightly-constrained configuration where deep layers $\hat \bW_{k+1}, ..., \hat \bW_{L}$ and $\hat \bW_{k+1}^{(i)}, ..., \hat \bW_{L}^{(i)}$ must implement the same mapping, in function space, from sub-network representation to network output. Our proof technique has the advantage that it never reasons in parameter space, but rather treats representations (projections onto singular directions, margins) that can be achieved with many parameter configurations. 

The proof of \Cref{stable sub network with stable rank} (a stable sub-network followed by a low-rank layer implies stability of the full network) already reveals that the low-rank bias is crucial to avoid destruction of stability in the deeper layers: indeed, otherwise the perturbation $\bW_k^\epsilon$ cannot be controlled. To make this intuition concrete, we designed a counterexample where a sub-network is stable, and yet the full network is unstable when it is not followed by a low-rank layer (\Cref{unstable network high rank}). However, finding \emph{any} unstable network which admits a stable sub-network is not enough to provide a meaningful counterexample. A core challenge with the endeavour is to establish that the crafted candidate solutions are indeed minimum-norm solutions, which is difficult in general, as mentioned above, since global minima in non-convex settings do not have any trivial distinguishing property. We thus engineered a toy setting where the norm can be bounded from below, and where the solution attains the bound. This was made possible by exploiting a key property of norm minimization, namely the equi-distribution of norm across layers. We think this proof technique is a new tool to explore configurations that lead to stable or unstable minimum-norm interpolants and gain insights into the mechanisms that favor stability.

%



In our analysis we also take a step towards addressing some limitations in the literature discussed earlier.
Algorithmically, interpolating neural networks studied here can be obtained as solutions of a GF and so, unlike previous results, our findings hold for $t \gg n$ and potentially $t \to \infty$.
In contrast to existing works, the analysis does not require regularity assumptions such as Lipschitzness or smoothness in the parameters, and crucially, parameter stability. 



%
%
%

\paragraph{Paper organization} Technical preliminaries and background are provided in \Cref{sec:preliminaries}. The main results are presented in \Cref{sec:main}, with the proofs detailed in \Cref{stable rank}. Additional related works, omitted proofs, and some experiments supporting our conjecture (that a trained deep neural network often consists of a stable sub-network and several final low-rank layers) are provided in the appendix.

%
%

\section{Preliminaries}
\label{sec:preliminaries}

\paragraph{Neural networks.}  We consider fully-connected neural networks with ReLU activations, $L$ layers and uniform width $d$ for hidden layers, no bias, and a real-valued output.
In this setting, network $\net: \mathbb{R}^{d_0} \rightarrow \mathbb{R}$ is defined by weight matrices $\{\mathbf{W}_{\ell}\}_{1 \leq \ell \leq L}$, one matrix per layer, with $\mathbf{W}_{\ell} \in \mathbb{R}^{d_\ell \times d_{\ell -1}}$,  $d_{\ell} = d$ for each $1 \leq \ell \leq L - 1$, and $d_L = 1$.
Network $\net$ computes its output on an input $\mathbf{x} \in \mathbb{R}^{d_0}$ by computing a \emph{pre-activation} vector $\mathbf{h}_\ell$ and a \emph{post-activation} vector $\mathbf{y}_\ell$ for each layer $\ell$ starting from $\mathbf{y}_0 = \mathbf{x}$ as follows, where the activation function $\phi_\ell$ is $\relu(x) = \max\{x, 0\}$ applied to vectors element-wise  for every layer $1 \leq \ell \leq L-1$ and $\phi_L$ is the identity function for layer $L$:
\begin{equation}\label{random nn}
    \mathbf{h}_{\ell}  =   \mathbf{W}_{\ell} \cdot \mathbf{y}_{\ell-1} \qquad \mathbf{y}_{\ell} = \phi_\ell (\mathbf{h}_{\ell}).
  \end{equation}
  We do not explicitly consider the bias term, however it can be modelled by appending $1$ to $\by_{\ell-1}$.  
  
  The network's output $\net(\mathbf{x})$ is given by $y_L \in \mathbb{R}$ and the weights $\mathbf{W}_L$ are referred to as the \emph{readout weights}.
  The collection of all parameters is denoted as  $\th$, 
  and we note $\net(\cdot) = \net(\, \cdot \, ; \, \th)$ to emphasize the dependency in the parameters.
Network $\net$ is an instance of a neural architecture $\mathbb{A} = \langle L, d, d_0 \rangle$ determined by the number of layers $L$, the width $d$ of the hidden layers and the input dimension $d_0$.

In the following we use notation $\net^{1:k-1} : \R^{d_0} \to \R^{d_{k-1}}$ to denote a neural network with parameters $\{\mathbf{W}_{\ell}\}_{1 \leq \ell \leq k-1}$  obtained from $\net$ by removing all layers with $\ell > k-1$.

A useful property of $\relu$  is \emph{positive homogeneity}: $\relu(\alpha x) = \alpha \, \relu(x)$ for each $x \in \mathbb{R}$ and $\alpha \geq 0$.
  In particular, we have $\net(\alpha \bx; \, \th) = \alpha \, \net(\bx; \, \th)$ and $\net(\bx; \, \alpha \th) = \alpha^L \, \net(\bx; \, \th)$.
\paragraph{Stable rank.} The Frobenius norm of a matrix $\mathbf{A} \in \mathbb{R}^{p \times q}$ with singular values $s_j$ for $1 \leq j \leq \mathrm{min}(p,q)$ is given by $\| \mathbf{A} \|_F = (\sum_{j} s_j^2)^{\frac12}$. The spectral norm of $\mathbf{A}$, that is $\|\bA\|_2$ is given by the largest absolute value of its singular values.
The stable rank  of matrix $\mathbf{A}$ is defined as $\mathrm{S}(\bA) = \|\bA\|_F / \|\bA\|_2$.
In particular, matrix $\mathbf{A}$ has stable rank $1$ if and only if its rank (defined in the usual way) is also one \citep{tropp2015introduction}; in this case, its singular value with largest absolute value is also the only non-zero singular value.
We will sometimes refer to the Frobenius norm of a neural network $\net$ as that of the matrix obtained by concatenating all weight matrices of $\net$.


\paragraph{Training set.} A training set  $(\mathbf{X}, \mathbf{y})$
consists of $n$ examples $(\mathbf{x}_i, y_i)$ sampled i.i.d.\ from an unknown distribution $p$ on $\cX \times \{-1,+1\}$ where the
input space $\cX$ is an Euclidean ball of radius one.
%
Furthermore, we denote by $(\mathbf{X}^{(j)}, \mathbf{y}^{(j)})$ the training dataset obtained from
$(\mathbf{X}, \mathbf{y})$ by re-sampling $j$-th example according to $p$ independently and we note $(\bx^{(j)}, y^{(j)})$ the new sample.
Finally, we say that $\net$ interpolates $(\mathbf{X}, \mathbf{y})$
if $\net (\mathbf{x}_i) = y_i$ for each $i \in [n]$.

\paragraph{Minimum-norm interpolating neural network.}
Given a neural architecture $\mathbb{A} = \langle L, d, d_0 \rangle$ and a training set $(\bX, \by)$, we assume that we have access to an algorithm $\cT$ which returns the parameters of a minimum-norm interpolating neural network instantiating $\mathbb{A}$. We then denote
\begin{align*}
  \cT(\bX, \by) = \argmin_{\th}\cbr{\|\th\|^2 ~:~ \forall i \in [n] \quad \net(\bx_i; \, \th) = y_i}.
\end{align*}
Finally, we denote the  parameters obtained by algorithm $\cT$ as $\hat \th = \cT(\bX, \by)$.

%

Training to interpolation and minimum-norm solutions are common assumptions which represent idealized views of gradient-based training algorithms (in particular with weight decay) \citep{timor2023implicit,galanti2023characterizing}. For example, in the limit of infinite training time, gradient flow on ReLU networks converges to a minimum-norm interpolant \citep{lyu2019gradient,ji2020directional}.

\paragraph{Algorithmic stability.}
Algorithm $\mathcal{A}$ is $\beta$-uniformly $\epsilon$-stable~\citep{kutin2002almost} with respect to a data distribution $p$ if, for each training set $(\mathbf{X}, \mathbf{y})$ sampled from $p$,
the following holds
for
$i \in [n]$, where $\hat \th = \cA(\bX, \by)$ and $\hat \th^{(i)} = \cA(\bX^{(i)}, \by^{(i)})$:
\[
  \P\left(
    \left| \net(\bx; \, \hat \th) - \net( \bx; \, \hat \th^{(i)}) \right| \leq \epsilon
  \right)
  \geq 1 - \beta
\]
  where $\P()$ is taken with respect to the jointly distributed $(\bX, \by, \bx^{(i)}, y^{(i)}, \bx, y)$.

  This notion of stability is weaker than the
  well-known notion of $\epsilon$-\emph{uniform stability}~\citep{bousquet2002stability}: $\sup_{\bX, \by, \bx, y, i} | \net(\bx; \, \hat \th) - \net( \bx; \, \hat \th^{(i)}) | \leq \epsilon$,
  however both coincide for $\beta = 0$ almost surely.
  Often, we will say that the algorithm is stable when for a fixed $\beta$, $\epsilon = \mathcal{O}_{n \to \infty} (n^{-\alpha})$ for some $\alpha > 0$.
  %
  We will occasionally abuse terminology and speak of a \emph{stable} minimum-norm interpolating neural network $\net (\, \cdot \,; \, \hat \th)$, meaning that algorithm $\cT$ generating $\hat \th$ is stable.

  \section{From sub-network stability to prediction stability}
  \label{sec:main}
In this section, we present our main results and discuss their implications. Our results depend on two technical assumptions and one main hypothesis, which we introduce next. The first assumption concerns the complexity of the learning problem~\citep{timor2023implicit}:
\begin{assumption}[$B$-admissible training set]
  \label{asm:data}
  Given a finite $B >0$, we call a training set $B$-admissible if
  there exists a neural network $\net$ of architecture $\ip{L^*, d, d_0}$ for some $L^* \geq 2$ with parameters $\th^*$ such that
  $\net(\bx_i; \, \th^*) = y_i$ for $i \in [n]$ and its weight matrices satisfy $\max_k \|\bW^*_k\|_F \leq B$.
\end{assumption}
Under this assumption, the training set can be viewed as generated by a hypothetical teacher network whose weight matrices have bounded norms.
%
Note that Assumption \ref{asm:data} is only meaningful in learning scenarios where data cannot be arbitrarily complex (otherwise, one can construct instances such that $B \to \infty$ as $n \to \infty$).
%
%


%
%
We next define our notion of sub-network.
\begin{definition}[Sub-network]
  \label{def:sub-network}
  Given $1 \leq k \leq L-1$, consider the following decomposition of weight matrix $\bW_k \in \mathbb{R}^{d \times d}$:
  $\bW_k  = \lambda_k \bu_k {\bv_k}\tp + \bW_k^{\epsilon}$ where $\lambda_k > 0$ is the leading singular value of $\bW_k$, $\bu_k, \bv_k$ are unitary vectors, $\bu_k$ is the leading left singular vector of $\bW_k$ and $\bv_k$ is the leading right singular vector of $\bW_k^T$.
  Then, a \textbf{sub-network} at position $k$ is defined as
  \begin{align}\label{eq:sub-network}
    \fsub_k(\bx; \, \th) := \bv_k\tp \net^{1:k-1}(\bx; \, \th)~.
  \end{align}
\end{definition}

In relation to the recent results from \citet{timor2023implicit}, we assume that our trained deep networks contain at least one layer with low stable rank.
\begin{assumption}[Low-rank bias]\label{existence low stable rank}
  Given $a >0$, $\epsilon = M / n^{-\alpha}$, for some $M \geq 0, \alpha > 0$ and a network $\net (\th, \cdot)$ for $\mathbb{A} = \langle L, d, d_0 \rangle$, there exists $1 \leq k \leq L-1$ such that the following holds:
  \begin{equation}\label{eq:low stable rank}
      \mathrm{S} (\bW_k) \leq 1 + a \, \epsilon ~.
  \end{equation}
\end{assumption}
Finally, we state our key hypothesis, namely the existence of a stable sub-network.
\begin{hypothesis}[$\beta$-uniformly $\epsilon$-stable sub-network]
  \label{asm:stable-subnetwork}
  Given $\beta \in [0,1]$ and $\epsilon = M / n^{-\alpha}$, for some $M \geq 0, \alpha > 0$, and $L \geq L^*$, with parameters $\hat \th$ and $\hat \th^{(i)}$ generated by algorithm $\cT$ for architecture $\mathbb{A} = \langle L, d, d_0 \rangle$, the following holds: 
    \begin{align*}
      \P\pr{\abs{ \fsub_k(\bx; \, \hat \th) - \fsub_k(\bx; \, \hat \th^{(i)})} \leq \epsilon} \geq 1 - \beta~.
    \end{align*}
  \end{hypothesis}
%
  
%
%
  %

Now we present our main results. First, we establish that if a network admits, under minimum-norm interpolation, a stable sub-network followed by a low-rank layer, then the network is stable (a complete statement of this result is given in \Cref{stable sub network with stable rank detailed}).

\begin{mainresult}[\emph{sketch}]\label{stable sub network with stable rank}
  Let $a >0$ and $\epsilon = n^{-\alpha}$ for some $\alpha > 0$ and
  suppose that there exist $(B, k, L, \beta)$ such that
  Assumptions \ref{asm:data}, \ref{existence low stable rank} and Hypthesis \ref{asm:stable-subnetwork} are satisfied.
  Then, assuming that sample size satisfies
  $n = \Omega(\max(a \, B^L, B^{L-k + 1}))$, there is a universal
  constant $C > 0$ such that $\cT$ is $\beta$-uniformly
  \begin{align*}
    C \, \big(1 + B^{2L -k +1} + a \, B^{3 L-k+1} \big) \, \epsilon \,\, \text{---} \,\, \text{stable}~.
  \end{align*}
\end{mainresult}
The main implication of the above result is that, with a low-rank layer, the stability of the entire minimum-norm interpolating neural network is controlled by the stability of its sub-network. 

%
The proof requires the key assumption that the stable sub-network is followed by a layer with a weight matrix of a low stable rank, which will preserve the signal as it propapagates into deeper layers.
%

It is actually possible to construct a counterexample of a deep network which admits a stable sub-network upon minimum-norm interpolation but which is unstable, when it is not followed by a low-rank layer. 
\begin{mainresult}\label{unstable network high rank}
    There exists a data distribution $p$ on $\cX \times \{+1, -1\}$, a neural architecture $\mathbb{A}$, $\epsilon = n^{-\alpha}$ for some $\alpha > 0$, and $(B, k, L, \beta)$ such that:
    \begin{itemize}
        \item Assumption \ref{asm:data} and Hypothesis \ref{asm:stable-subnetwork} hold; and
        \item there exists a constant $C > 0$ such that with probability at least $1 - \beta$ the following holds

    \[
  \left| \net(\bx; \, \hat \th) - \net( \bx; \, \hat \th^{(i)}) \right| > C
\]
    \end{itemize}
\end{mainresult}

These results combined suggest an important interplay between stability and low-rank bias. In the configuration of \Cref{stable sub network with stable rank}, the stable sub-network effectively exhibits deep neural collapse \citet{NEURIPS2023_a60c43ba}, a phenomenon observed in gradient-based training of deep neural networks. Low-rank layers would, in this scenario, play the role of preserving the stability through the layers following the deep neural collapse. Recently \citet{timor2023implicit} showed that the stable rank decays roughly at the rate $B^{\frac{L}{k}}$ where $B$ can be interpreted as a problem complexity term (see Assumption \ref{asm:data}).\footnote{In fact, they showed a slightly stronger upper bound on the harmonic mean: $\frac{k}{\sum_{j=1}^k(1/\mathrm{S}(\widehat \bW_j))} \leq B^{\frac{L}{k}}$.}
In our case, the stable rank decay rate is assumed to be $a \, \epsilon$ with a free parameter $a \geq 0$.

Note that, in \Cref{stable sub network with stable rank}, while the overall stability scales linearly with the
stability of the sub-network, the bound is attenuated by a
$B$-dependent factor.  One extreme (pessimistic) case is $a \gg 0$,
that is when weight matrices are sufficiently far from rank-one.
Then, the factor in the worst case becomes of order $B^{3 L}$.
Intuitively, this occurs because, learning all perturbation matrices
$(\bW_k^{\epsilon})_k$ (see \Cref{def:sub-network})
becomes unavoidable.
In another extreme case of rank-one matrices and a shallow stable
sub-network ($a=0, k=2$), we have an overall stability bound of order
$B^{2 L} \epsilon$.
%
%

In a more optimistic scenario, $a$ is sufficiently small such that
weight matrices have a stable-rank close to one.  In this case we have a
dominant factor $B^{2 L-k+1}$, which captures the cost of training a
neural network that is deeper than the sub-network.  In particular,
observe that the deeper the stable sub-network is (increasing $k$),
the smaller the cost.  For a deep stable sub-network ($k=L-1$), we
reach a factor $B^{L}$ which is similar to dimension-free analysis of
deep ReLU neural networks~\citep{golowich2018size}.

\subsection{Proof of the main results}\label{stable rank}
%



\subsubsection{Stability with a low-rank layer}

We start by restating more precisely \Cref{stable sub network with stable rank}.

\begin{theorem}\label{stable sub network with stable rank detailed}  
  Suppose that datasets are $B$-admissible according to Assumption \ref{asm:data}. Let $\epsilon = M / n^{-\alpha}$, for some $M \geq 0, \alpha > 0$ and let $a > 0$, consider $L \geq L^*$ and $1 \leq k \leq L-1$ that satisfy \eqref{eq:low stable rank}.
  Suppose that the sub-network is $\beta$-uniformly $\epsilon$-stable according to Hypothesis \ref{asm:stable-subnetwork}.
  Then, assuming that the sample size satisfies
  $
  n \geq \max\pr{a \cdot
    2 M  B^L , \, 4 M   B^{L-k + 1}
  }^{\frac{1}{\alpha}}
  $,
  $\mathcal{T}$ is $\beta$-uniformly $\epsilon'$-prediction stable with
  \begin{align*}
    \epsilon' = &
    (\cy + 8 \cymult  B^{2L-k+1}) \, \epsilon
    \\& +
    2 a \pr{B^L + B^{2L-k+1} \epsilon + 4 B^{3L-k+1}} \, \epsilon      
    ~.
  \end{align*}
  \end{theorem}
  
  The proof relies on Assumption \ref{existence low stable rank}, \Cref{low stable rank preserve signal,margin stable rank}.
First, we discuss high-level proof ideas.
\paragraph{Some proof ideas}
For simplicity consider a rank-one case ($a=0$).
The key \Cref{low stable rank preserve signal} shows that if a deep ReLU network interpolates the data, then the prediction done at the rank-one layer is maintained throughout the rest of the layers.
Consider a decomposition $\bW_k = \lambda_k \bu_k \bv_k\tp$, where $\bv_k$ defines a hyperplane which separates inputs in the feature space of the previous layer given by $\net^{1:k-1}(\, \cdot \,, \hat \th)$.
Then, multiplication by $\lambda_k \bu_k$ and the propagation through all subsequent layers is just a rescaling of the predictions to fit the labels.
In other words, the sign of $\fsub_{k}(\bx) = \bv_k\tp \net^{1:k-1}(\bx, \hat \th)$ for each $\bx$ already determines the final output and the prediction is done at the sub-network level (see \cref{fig:compression} in \cref{app:experiments} for further discussion).

Next, in \Cref{margin stable rank}, we show that if a deep ReLU network is a minimum-norm interpolant of the data, the intermediate prediction of data points at the sub-network level must have substantial margin $\gamma$.
Intuitively, if the margin were infinitesimally small, this would require a large contribution in (one of the weights of) the subsequent layers to compensate and yield an output of order 1, which is forbidden by the fact that the network has minimum norm.





%

\begin{lemma}\label{low stable rank preserve signal}
  Consider a subnetwork $f_k$ as defined in \Cref{def:sub-network}.  
  Then, there exists a bounded function $b(\,\cdot\, ; \, \th)$,
  and $C_{\th}^+, C_{\th}^+ > 0$ (independent from the input),
  such that, for all $\bx \in \mathbb{R}^{d_0}$ the following is true:
  \begin{enumerate}
  \item If $\fsub_k(\bx; \, \th) > 0$, then
    $\net(\bx) + b (\bx; \, \th) \cdot \epsilon = C_{\th}^+ \cdot \fsub_k(\bx; \, \th)$~.
  \item If $\fsub_k(\bx; \, \th) \leq 0$, then
    $\net(\bx) +  b (\bx; \, \th) \cdot \epsilon = C_{\th}^- \cdot \fsub_k(\bx; \, \th)$~.
  \item For all $\bx \in \mathbb{R}^{d_0}$, $|b (\bx; \, \th)| \leq a \, \pr{\prod_{j=1}^L \|\bW_j\|_2}$~.
    \item $\left\| \bW_k^{\epsilon} \; \net^{1:k-1}(\bx) \right\| \leq a \, \epsilon \, \prod_{j=1}^k \|\bW_j\|_2$~.
  \end{enumerate}
  %
\end{lemma}




\begin{lemma}\label{margin stable rank}
  Consider a subnetwork $f_k$ as defined in \Cref{def:sub-network}.
  Suppose that datasets are $B$-admissible according to Assumption \ref{asm:data}.
  Let $\hat \th = \cT(\bX, \by)$ and $\hat \th^{(i)} = \cT(\bX^{(i)}, \by^{(i)})$ and assume that $|y_i| \geq 1$ for all $i \in [n]$.    
  Then, assuming that the sample size satisfies $n \geq \left(2 \cdot M \cdot a \cdot B^L\right)^{\frac{1}{\alpha}}$,
  for any $\gamma \leq 1 / (2 \cdot B^{L-k + 1})$, and $\mu \geq B^{k-1}$,
  \begin{align*}
    \gamma \leq \left|\fsub_k(\bx_i; \, \hat \th) \cdot y_i \right| \leq \mu \qquad (\forall i \in [n]).
  \end{align*}
\end{lemma}
  
\begin{proof}[Proof of \Cref{stable sub network with stable rank}]
In the following let $(\mathbf{x}, y)$ be the point replacing the $i$th example in the training set $(\bX, \by)$ and so by interpolation we have $ y = \net(\bx ; \, \hat \th^{(i)})$.

By \Cref{low stable rank preserve signal} there exist $C_{\hat \th}, C_{\hat \th^{(i)}}$ independent from the input and $b(\,\cdot\,;\, \hat \th)$, $b(\,\cdot\,;\, \hat \th^{(i)})$ such that
\begin{align}\label{eq:stable_sub_network_srank_1}
  &\net(\bx; \, \hat \th) + b (\bx; \, \hat \th) \cdot \epsilon = C_{{\hat \th}} \cdot \fsub_k(\bx; \, {\hat \th})~,
  \\&
  \net(\bx; \, \hat \th^{(i)}) + b (\bx; \, \hat \th^{(i)}) \cdot \epsilon = C_{{\hat \th^{(i)}}} \cdot \fsub_k(\bx; \, {\hat \th^{(i)}})~.
\end{align}
Observe that,
\begin{align*}
  &\left| \net(\bx; \, \hat \th) - \net(\bx; \, \hat \th^{(i)}) \right|
    -
    \abs{\pr{b (\bx; \, \hat \th)  - b (\bx; \, \hat \th^{(i)})}\cdot \epsilon}\\
  &\quad \leq \left| \net(\bx; \, \hat \th) - \net(\bx; \, \hat \th^{(i)}) + (b (\bx; \, \hat \th)  - b (\bx; \, \hat \th^{(i)}))\cdot \epsilon \right|\\
      &\quad \leq \left| C_{\hat \th} \cdot  \fsub_k(\bx; \, \hat \th) - C_{\hat \th} \cdot \fsub_k(\bx; \, \hat \th^{(i)}) \right| \\&\quad+  \left|C_{\hat \th} \cdot \fsub_k (\bx; \, \hat \th^{(i)}) - C_{\hat \th^{(i)}} \cdot \fsub_k (\bx; \, \hat \th^{(i)}) \right| \\
  &\quad \leq \left| C_{\hat \th} \right | \cdot  \left| \fsub_k(\bx; \, \hat \th) - \fsub_k(\bx; \, \hat \th^{(i)}) \right| \\&\quad+ \left| \fsub_k (\bx; \, \hat \th^{(i)}) \right| \cdot \left|C_{\hat \th} - C_{\hat \th^{(i)}}  \right|\\
  &\quad \leq \left| C_{\hat \th} \right | \cdot  \epsilon + \left| \fsub_k (\bx; \, \hat \th^{(i)}) \right| \cdot \left|C_{\hat \th}  - C_{\hat \th^{(i)}}  \right| \tag{By sub-network stability assumption}\\
    &\quad \leq \left| C_{\hat \th} \right | \cdot  \epsilon + \mu \cdot \left|C_{\hat \th}  - C_{\hat \th^{(i)}}  \right|~.
      \tag{By \cref{margin stable rank}}
\end{align*}
First note that by \Cref{low stable rank preserve signal},
\begin{align*}
  |b (\bx; \, \hat \th)| \leq a \prod_{j=1}^L \|\widehat \bW_j\|_2 \leq a \, B^L =: b
\end{align*}
where $\|\widehat \bW_j\|_2 \leq B$ comes as in the first part of the proof of \Cref{margin stable rank}.
Similarly, $|b (\bx; \, \hat \th^{(i)})| \leq b$.
Then,
\begin{align*}
  \abs{\pr{b (\bx; \, \hat \th)  - b (\bx; \, \hat \th^{(i)})}\cdot \epsilon}
  \leq 2 \, a \, \epsilon \, B^{L}~.
\end{align*}
Now, the idea is to use \cref{eq:stable_sub_network_srank_1} to control $|C_{\hat \th}|$ and to control $ |C_{\hat \th}  - C_{\hat \th^{(i)}}  |$ in terms of difference of sub-network predictions and invoke sub-network stability once again.
  W.l.o.g., let $i \neq 1$ and so
  \begin{align*}
    C_{\hat \th}
    &=
    \frac{\net(\bx_1; \, \hat \th) + b (\bx_1; \, \hat \th) \cdot \epsilon}{\fsub_k(\bx_1; \, \hat \th)}
    =
      \frac{y_1 + b (\bx_1; \, \hat \th) \cdot \epsilon}{\fsub_k(\bx_1; \, \hat \th)}~,\\
    C_{\hat \th^{(i)}}
    &=
    \frac{\net(\bx_1; \, \hat \th^{(i)}) + b (\bx_1; \, \hat \th^{(i)}) \cdot \epsilon}{\fsub_k(\bx_1; \, \hat \th^{(i)})}
    =
    \frac{y_1 + b (\bx_1; \, \hat \th^{(i)}) \cdot \epsilon}{\fsub_k(\bx_1; \, \hat \th^{(i)})}~.
  \end{align*}
  So, by \Cref{low stable rank preserve signal} and \Cref{margin stable rank} (with $\gamma$ defined therein),
  \begin{align*}
    |C_{\hat \th}|
    \leq
      \frac{|y_1| + b \cdot \epsilon}{|\fsub_k(\bx_1; \, \hat \th)|}
    \leq
      \frac{|y_1| + b \cdot \epsilon}{\gamma}~.
  \end{align*} 
  Now we turn our attention to the gap:
  \begin{align*}
    \left|C_{\hat \th}  - C_{\hat \th^{(i)}}  \right|
    &=
    \abs{
    \frac{y_1 + b (\bx_1; \, \hat \th) \cdot \epsilon}{\fsub_k(\bx_1; \, \hat \th)}
    -
    \frac{y_1 + b (\bx_1; \, \hat \th^{(i)}) \cdot \epsilon}{\fsub_k(\bx_1; \, \hat \th^{(i)})}
      }\\
    &\leq
      |y_1| \, \abs{
    \frac{\fsub_k(\bx_1; \, \hat \th) - \fsub_k(\bx_1; \, \hat \th^{(i)})}{\fsub_k(\bx_1; \, \hat \th) \cdot \fsub_k(\bx_1; \, \hat \th^{(i)})}
      }
      \\&\quad+
      \abs{
      \frac{b (\bx_1; \, \hat \th)}{\fsub_k(\bx_1; \, \hat \th)}
      -
      \frac{b (\bx_1; \, \hat \th^{(i)})}{\fsub_k(\bx_1; \, \hat \th^{(i)})}
      } \cdot \epsilon\\
    &\stackrel{(a)}{\leq}
    \frac{|y_1| \cdot \epsilon}{\abs{\fsub_k(\bx_1; \, \hat \th) \cdot \fsub_k(\bx_1; \, \hat \th^{(i)})}}
      \\&\quad+
      \frac{b \cdot \epsilon}
      {\abs{\fsub_k(\bx_1; \, \hat \th) \cdot \fsub_k(\bx_1; \, \hat \th^{(i)})}}\\
    &\stackrel{(b)}{\leq}
      \frac{2 \, (|y_1| + b) \, \epsilon}{(\gamma - \epsilon)_+^2 + \gamma^2 - \epsilon^2}
  \end{align*}
  where in step $(a)$ we used the assumption that sub-network is $\epsilon$-stable and \Cref{low stable rank preserve signal}, while step $(b)$ amounts to lower-bounding the denominator, and deferred till the end of the proof.

Putting all together we have
  \begin{align*}
    \left| \net(\bx; \, \hat \th) - \net(\bx; \, \hat \th^{(i)}) \right|
    \leq &
    2 \cdot b \cdot \epsilon
    +
    \frac{\cy + b \cdot \epsilon}{\gamma} \cdot \epsilon
    \\&\quad +
    \mu \cdot \frac{2 \, (\cy + b) \, \epsilon}{(\gamma - \epsilon)_+^2 + \gamma^2 - \epsilon^2}~.
  \end{align*}
  Now, let's require $\epsilon \leq \gamma / 2$:
  By assumption we have that $\epsilon = (M/n)^{\alpha}$ and so our requirement is equivalent to
  $n \geq \left((2 M)/\gamma\right)^{\frac{1}{\alpha}}$.
  In overall, this gives
  \begin{align*}
    \left| \net(\bx; \, \hat \th) - \net(\bx; \, \hat \th^{(i)}) \right|
    &\leq
    (2 b + \cy) \epsilon
      + \frac{b}{\gamma} \, \epsilon^2
    \\&\quad+
      2 \, (\cy + b) \, \frac{\mu}{\gamma^2} \, \epsilon~.
  \end{align*}
  We conclude by using the bounds on $\gamma$ and $\mu$ in \cref{margin stable rank}, and the one on $b$ which comes by \cref{low stable rank preserve signal}.
%

  \paragraph{Proof of step $(b)$}\label{eq:proof_of_step_star}
   By \Cref{margin stable rank} we have that
$
\abs{\fsub_k(\bx; \, \hat \th^{(i)}) }\geq \gamma
$.
Now, given the above, the fact that the sub-network is $\epsilon$-stable, and triangle inequality we have
\begin{align*}
  \epsilon
  &\geq \abs{ \fsub_k(\bx; \, \hat \th) -\fsub_k(\bx; \, \hat \th^{(i)}) }
  \\&\geq 
    \abs{\fsub_k(\bx; \, \hat \th^{(i)}) }
    -
    \abs{ \fsub_k(\bx; \, \hat \th)}
  \geq
    \gamma
    -
    \abs{ \fsub_k(\bx; \, \hat \th)}~.
\end{align*}
This gives
\begin{align*}
  &\pr{ \fsub_k(\bx; \, \hat \th) -\fsub_k(\bx; \, \hat \th^{(i)}) }^2 \leq \epsilon^2\\
  \Longleftrightarrow \qquad
  &\fsub_k(\bx; \, \hat \th)^2 + \fsub_k(\bx; \, \hat \th^{(i)})^2 - \epsilon^2
    \\&\qquad\leq 2 \, \fsub_k(\bx; \, \hat \th) \cdot \fsub_k(\bx; \, \hat \th^{(i)})\\
  \Longrightarrow \qquad
  &\frac{(\gamma - \epsilon)_+^2 + \gamma^2 - \epsilon^2}{2}
    \leq \fsub_k(\bx; \, \hat \th) \cdot \fsub_k(\bx; \, \hat \th^{(i)})~.
\end{align*}
\end{proof}

\subsubsection{Unstability without low-rank layer}

Next, we prove \Cref{unstable network high rank} by explicitly constructing a counterexample: a stable sub-network followed by a non low-rank layer, resulting in an unstable network. To guarantee that our counterexample is obtained through minimum-norm interpolation, we rely on the following lemma, which gives a template of minimum-norm solution in a specific toy data setting. 

\begin{lemma}\label{lem:min-norm-sol-toy}
    Consider the neural architecture $\mathbb{A} = \langle 3, 2, 2 \rangle$. Consider a training set on $\cX \times \{-1, + 1\}$ of two data points $(x, 0), +1$ and $(0,z), -1$ with $x > 0$ and $z > 0$. A minimum-norm interpolant is given by:
    $\hat \bW_1 = \begin{pmatrix}
        (\frac{1}{x})^{1/3} \; 0 \\
        0 \; (\frac{1}{z})^{1/3}
    \end{pmatrix}$, $\hat \bW_2 = \begin{pmatrix}
        (\frac{1}{x})^{1/3} \; 0 \\
        0 \;(\frac{1}{z})^{1/3}
    \end{pmatrix}$ and $\hat \bW_3 = \begin{pmatrix}
        (\frac{1}{x})^{1/3} \; -(\frac{1}{z})^{1/3}
    \end{pmatrix}$
\end{lemma}

\begin{proof}[Proof of \Cref{unstable network high rank}]
    Consider the following data distribution $p (\bx, y)$ on $\cX \times \{+1, -1\}$: with probability $\frac{1}{2}$, $\bx, y = (0, \frac{1}{10}), -1$, with probability $\frac{1}{4}$, $\bx, y = (\frac{1}{2}, 0), +1$ and with probability $\frac{1}{4}$, $\bx, y = (\frac{1}{4}, 0), +1$. 

    Consider the neural architecture $\mathbb{A} = \langle 3, 2, 2 \rangle$. Consider a training set with two data points $\bx^+ = (\frac{1}{2}, 0)$ and $\bx^- = (0, \frac{1}{10})$. By Lemma \ref{lem:min-norm-sol-toy}, a minimum-norm solution is given by: 
    $\hat \bW_1 = \begin{pmatrix}
        2^{1/3} \; 0 \\
        0 \; 10^{1/3}
    \end{pmatrix}$, $\hat \bW_2 = \begin{pmatrix}
        2^{1/3} \; 0 \\
        0 \;10^{1/3}
    \end{pmatrix}$ and $\hat \bW_3 = \begin{pmatrix}
        2^{1/3} \; -10^{1/3}
    \end{pmatrix}$. 
    
    The singular value decomposition of $\hat \bW_2$ is immediate and picking the dominating right singular vector, we have that the sub-network is given by: $f_2 (\hat \theta, \bx) = 10^{1/3} x_2$. 
    
    Assume we resample the data point $\bx^+$: with probability $\frac{1}{4}$, we have $\bx_+^{(1)} = (\frac{1}{4}, 0)$. By Lemma \ref{lem:min-norm-sol-toy}, a minimum-norm interpolant is given by:
    $\hat \bW_1^{(1)} = \begin{pmatrix}
        4^{1/3} \; 0 \\
        0 \; 10^{1/3}
    \end{pmatrix}$, $\hat \bW_2^{(1)} = \begin{pmatrix}
        4^{1/3} \; 0 \\
        0 \; 10^{1/3}
    \end{pmatrix}$ and $\hat \bW_3^{(1)} = \begin{pmatrix}
        4^{1/3} \; -10^{1/3}
                                           \end{pmatrix}$.
                                          
The dominating right singular vector remains the same, so $f_2 (\hat \theta, \bx) = f_2 (\hat \theta^{(1)}, \bx) = 10^{1/3} x_2$. This implies that the sub-network is stable (specifically, it is $\frac{3}{4}$-uniformly $\epsilon$-prediction stable, for any $\epsilon > 0$). However, with probability $\frac{1}{4}$, $\bx = (\frac{1}{2}, 0)$, which leads to: 
    $\net (\hat \theta, \bx) = 1$ and $\net (\hat \theta^{(1)}, \bx) = \frac{1}{2} \cdot 4^{1/3} \cdot 4^{1/3} \cdot 4^{1/3} = 2$. Thus the network $N$ is unstable. 
\end{proof}

We conclude by proving the technical lemma establishing that our candidate solutions are minimum-norm solutions. The proof requires the following known result:

\begin{lemma}[{\citet[Lemma 14]{timor2023implicit}}]
  \label{lem:lemma14}
  Let $\hat \th = \cT(\bX, \by)$.
  Then, for every $1 \leq i < j \leq L$ we have $\|\widehat \bW_i\|_F = \|\widehat \bW_j\|_F$.
\end{lemma}

\begin{proof}[Proof of \Cref{lem:min-norm-sol-toy}]
    The proof relies on the fact that inputs $(x,0)$ and $(0,z)$ do not propagate using the same columns of $\bW_1$. We will prove that the (squared) norm of the first column $\|\bW_1 \mathbf{e}_1\|^2_2$ (where $\mathbf{e}_1=(1,0)^\top$) is lower bounded by $(\frac{1}{x})^{2/3}$. The same proof will give us that the (squared) norm of the second column $\|\bW_1 \mathbf{e}_2\|^2_2$ is lower bounded by $(\frac{1}{z})^{2/3}$. Therefore, the (squared) norm $\|\bW_1\|^2_2$ is lower bounded by $(\frac{1}{x})^{2/3} + (\frac{1}{z})^{2/3}$. Our candidate solution attains this norm. Then, we invoke Lemma 14 in \citet{timor2023implicit} (norm is equally distributed across layers at minimum-norm solutions) to conclude that there is no solution with smaller norm than our candiate, making it a minimum-norm interpolant.

    Let us now prove that $\|\bW_1 \mathbf{e}_1\|^2_2 \geq (\frac{1}{x})^{2/3}$.
    Let $\mathbf{y}_1:=\relu(\bW_1 (x,0))\in\mathbb{R}_+^2$ and $\mathbf{y}_2:=\relu(W_2 \mathbf{y}_1)\in\mathbb{R}_+^2$.
Then
\[
1 = N(x,0) = \bW_3^\top \mathbf{y}_2 \;\le\; \|\bW_3\|_2\,\|\mathbf{y}_2\|_2
\]
by Cauchy--Schwarz.
Moreover,
\[
\|\mathbf{y}_2\|_2 = \|\relu(\bW_2 \mathbf{y}_1)\|_2 \;\le\; \|\bW_2 \mathbf{y}_1\|_2
\;\le\; \|\bW_2\|_F\,\|\mathbf{y}_1\|_2.
\]
Finally,
\[
\|\mathbf{y_1}\|_2 = \|\relu(\bW_1(x,0))\|_2
\;\le\; \|\bW_1(x,0)\|_2
= x\,\|\bW_1 e_1\|_2.
\]
Combining the above inequalities yields
\[
\frac{1}{x} \;\le\; \|\bW_3\|_2\,\|\bW_2\|_F\,\|\bW_1 \mathbf{e}_1\|_2,
\].

Let
\[
a := \|\bW_1 \mathbf{e}_1\|_2, \qquad
b := \|\bW_2\|_F, \qquad
c := \|\bW_3\|_2.
\]
By the above, we have $abc \ge \frac{1}{x}$.
Applying the arithmetic--geometric mean inequality to $a^2,b^2,c^2$ yields
\[
\frac{a^2+b^2+c^2}{3}
\;\ge\;
(a^2 b^2 c^2)^{1/3}
=
(abc)^{2/3}
\;\ge\;
(\frac{1}{x})^{2/3}.
\]
Multiplying both sides by $3$ gives:
\[
\|\bW_1 \mathbf{e}_1\|_2^2 + \|\bW_2\|_F^2 + \|\bW_3\|_2^2
\;\ge\;
3\, (\frac{1}{x})^{-2/3}.
\]

Solving the minimum-norm interpolation with only one data point $(x, 0), +1$ is less restrictive than with both data points so the norm of the solution is smaller than the minimum-norm solution with both data points. Applying Lemma 14 from \citet{timor2023implicit} to the minimum-norm solution with only one data point ensures that $\|\bW_1 \mathbf{e}_1\|_2^2 \geq (\frac{1}{x})^{-2/3}$ since the norm is equally distributed across layers. This holds for the minimum-norm solution with only one data point, therefore it also holds for the the minimum-norm solution of interest (with both data points). We can conclude as exposed above that our candidate is a minimum-norm interpolant.
\end{proof}

\section{Conclusions, limitations, and future work}
In this work, we have studied sufficient conditions for algorithmic
stability in minimum-norm interpolating deep ReLU neural
networks. This study opens up several interesting avenues for future
research. One of the sufficient conditions we examined is the
existence of a stable sub-network, which we did not prove
theoretically and which remains an open question. We have tested the limits of this sufficient condition, finding that low-rank layer is a key additional assumption needed for stability to be guaranteed in this scenario, revealing strong interplay between stability and low-rank bias. In related areas,
such as the lottery ticket hypothesis, the existence of a
well-performing sub-network has been shown theoretically, which might
inspire the use of similar techniques to prove the existence of a stable sub-network. 
Finally, an interesting open question is
bridging the gap between the stability of minimum-norm interpolation
in neural networks under GF and actual optimization algorithms, such as
SGD. While there are several promising directions, a complete picture
incorporating stability analysis might require additional arguments~\citep{poggio2020theoretical,elkabetz2021continuous}.

One possible limitation of our analysis is that, even in the favourable scenario, the stability bound
involves a $B$-dependent factor which scales as $B^{2L-k+1}$ in the
best case.  For the case $k=L-1$, it might be possible to achieve a
better factor $B^L$ which would be in line with  Rademacher
complexity analysis~\citep{golowich2018size}.
%










\bibliography{learning}

\newpage
\appendix
\onecolumn

\section{Further implications of our results}

\paragraph{Stability implies generalization.}
  Let $f  : \R^2  \to [0, M]$ be a fixed known loss function.
  Then the \emph{risk} and the \emph{empirical risk} of the predictor parameterized by $\th$ are respectively defined as
  $R(\th) = \int f(\net(\bx; \, \th), y) \diff p(\bx, y)$
  and
  $\hat R(\th) = \frac1n \sum_{i=1}^n f(\net(\bx_i; \, \th), y_i)$.
  It is known that if $\hat \th$ is generated by an $\epsilon$-uniformly-stable algorithm $\cA$,
  there exist universal constants $c_1, c_2 > 0$ such that for any $\delta \in (0,1)$~\citep{feldman2018generalization,bousquet2020sharper},
  \begin{align}
    \label{eq:stab2gen}
    & \P\pr{|R(\hat \th) - \hat R(\hat \th)|
    \leq 
   c_1 \ln(n) \ln(1/\delta) \, \epsilon
    +
    c_2 \, M \sqrt{\frac{\ln(1/\delta)}{n}}
    } \\& \geq 1 - \delta~.
  \end{align}
  Hence, the gap between the population loss and empirical loss is controlled with high probability as long as the algorithm is stable.

Moreover, stability analysis provides valuable insights beyond generalization, such as controlling the variance of the algorithm, which is crucial for uncertainty quantification methods like bootstrapping ~\citep{elisseeff2005stability}.

\paragraph{Applications.}

The stability bound we presented can thus used in the following applications.
In particular, when combined with \cref{eq:stab2gen}, it implies a
high-probability bound on the generalization error:
\begin{align*}
  &|R(\hat \th) - \hat R(\hat \th)|
  \\&=
  \tilde{\mathcal{O}}\pr{
  \big(1 + B^{2L -k +1} + a \, B^{3 L-k+1} \big) \, \epsilon
  +
  \frac{1}{\sqrt{n}}
  }~.
\end{align*}
Under assumption $\epsilon = n^{-\alpha}$ for
$\alpha > 0$, the generalization error converges $0$ as $n \to \infty$
in probability.

Finally, the stability bound can also be used to give a bound on the
variance of a trained neural network, which is not normally achievable
through the uniform convergence bounds.  Controlling the variance of
trained predictors is interesting in the context of uncertainty
quantification (e.g., through ensembling, as discussed
in~\citep{elisseeff2005stability}).  In particular, Efron-Stein
inequality~\citep{boucheron2013concentration} implies that
\begin{align*}
  \Var(\net(\hat \th))
  \leq  
  \frac{C^2}{2} \,
  \big(1 + B^{2L -k +1} + a \, B^{3 L-k+1} \big)^2 \, n \, \epsilon^2~.
\end{align*}
Once again, assuming $\epsilon = n^{-\alpha}$ it turns out that the variance converges to $0$ asymptotically only if $\alpha > 1/2$, that is when sub-network is very stable with $\epsilon < 1/\sqrt{n}$.

\begin{figure}
  \centering  
        \includegraphics[width=0.5 \textwidth]{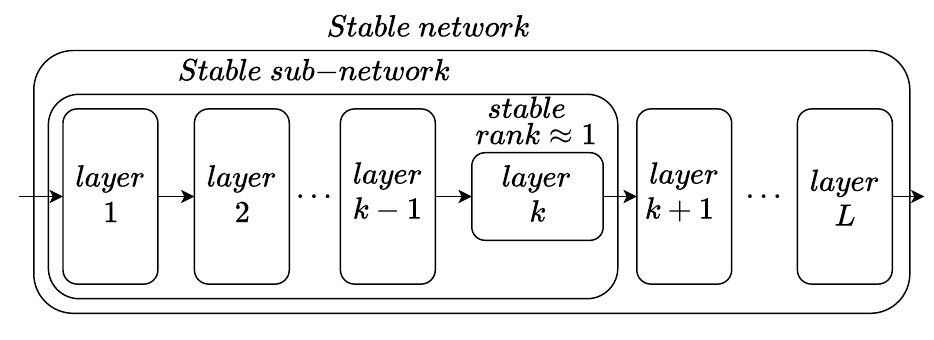}
        \caption{\textbf{A diagram of our main results.} Our main results revolve around three arguments: a) the data is expressible by a neural network with bounded weight matrix norm (Assumption \ref{asm:data}), b) the minimum-norm interpolating ReLU neural network contains at least one layer with a low stable rank matrix (see Assumption \ref{existence low stable rank})
          and c) the sub-network is stable (Hypothesis \ref{asm:stable-subnetwork}). In this scenario, we show in Theorem \ref{stable sub network with stable rank} that the full network is also stable. If b) does not hold (no low-rank layer), the full network may be unstable, even with a stable sub-network (\Cref{unstable network high rank}).
        }\label{fig:diagram}        
      \end{figure}

\section{Additional related work}
\label{sec:additional-related}

\paragraph{Generalisation and minimum-complexity interpolation.}
Empirical evidence strongly suggests that the  generalization ability of SGD algorithms remains high despite overparameterization in non-convex settings. These observations appear to be in conflict with the classical learning-theoretic \emph{complexity-fit tradeoff} viewpoint~\citep{Bousquet2004,zhang2021understanding}. 

In recent years, there has been growing interest in the hypothesis that the favorable generalization performance of optimization algorithms can be accounted for by the algorithm's \emph{implicit bias} or inherent model capacity control.
A prominent idea in this setting is that of \emph{minimum-complexity interpolation}, which states that the algorithm finds the simplest solution among those achieving  zero (or near-zero) training error.
While this phenomenon is well established in the context of linear regression (e.g.\ a pseudo-inverse solution to the least-squares problem recovers parameters with a minimum $L_2$ norm), the suggestion that it also applies to the training of a deep overparameterized ReLU neural networks is only rather recent.
As mentioned in the main text, GF on deep ReLU networks asymptotically converges to minimum-$L_2$-norm interpolants \citep{lyu2019gradient,ji2020directional,phuong2020inductive}. 
GF  is thought to be a good proxy for gradient descent algorithms because the approximation error introduced by discretization remains controlled under certain regularity assumptions~\citep{elkabetz2021continuous}.
Recent works have attempted to reconcile minimum-norm interpolation  with classical generalization theories based on the uniform convergence principle such as Rademacher complexity~\citep{bartlett2002rademacher,bartlett2017spectrally}.
Indeed, dimension-free bounds have provided some explanation as to why we do not observe overfitting in the absence of label noise \citep{ji2019polylogarithmic,telgarsky2022feature}, and observe only \emph{benign overfitting} otherwise~\citep{tsigler2020benign,koehler2021uniform,frei2023benign}.

\paragraph{Algorithmic stability}

The stability of interpolating kernel least-squares has been studied by \cite{rangamani2023interpolating} who established a connection to stability of the pseudo-inverse, which is indeed a minimum-norm interpolation, and which is in turn controlled by the smallest non-zero eigenvalue of the kernel matrix.
In this paper we look into a rather different setting of interpolation with neural networks rather than kernel machines, and provide sufficient conditions for such stability.

Recently, \cite{schliserman2022stability} analyzed the performance of gradient descent type methods on convex learning problems given  linearly separable data.
The setting of our paper can be interpreted as   noise-free labels and so in case of classification, we have separability as well, however the decision boundary can be non-linear (as we can conclude from Assumption \ref{asm:data}).

\paragraph{Neural collapse}
Recently, a phenomenon called neural collapse (NC)~\citep{papyan2020prevalence, han2021neural} has been observed where the post-activations (and weights of a final layer) of a well-trained neural network, appear to be clustered in a low-dimensional subspace.
Theoretical studies on the cause of NC mostly rely on unconstrained feature models \citep{fang2021layer,mixon2022neural} where the product of two matrices is trained under gradient flow: the two matrices train to a low-rank structure where their ranks equal the dimension of the outputs.

\paragraph{Lottery ticket hypothesis}

The existence of a small good neural network within a large deep neural network has been proposed before as a prominent hypothesis in neural network learning.
In particular the \emph{lottery ticket hypothesis}~\citep{frankle2018lottery} posits that deep neural networks contain small sub-networks which could be trained in isolation and lead to comparable performance, and that these sub-networks happen to be favored by standard initializations
The \emph{lottery ticket hypothesis} has however only been analyzed in some restricted settings \citep{frankle2020linear,malach2020proving,orseau2020logarithmic,sakamoto2022analyzing}. It is still an active area of theoretical research.
In this paper we explore a related concept, where the quality of the sub-network is captured by its stability, which establishes a link to analysis of the generalization error.

\paragraph{Benign overfitting.}
One prominent line of recent work that tries to explain the success of
deep learning is \emph{benign
  overfitting}~\citep{tsigler2020benign,koehler2021uniform}. The idea
behind benign overfitting is that even in the presence of
low-to-moderate noise, interpolating predictors are able to achieve
performance that is close to the noise rate (in other words, some
inductive bias is present in the interpolation procedure). In
particular, these works study this phenomenon through analysis of the
excess risk of interpolants (such as linear interpolants in high
dimension). The key idea behind these analyses is that
\emph{minimum-norm} interpolation is a sufficient condition for benign
overfitting. On the other hand, it is also well-known that uniform
stability is sufficient for uniform convergence and so one can show
that the excess risk has correct asymptotic behavior (convergence to
the noise rate)~\citep{shalev2010learnability}.

These theories appear to be at odds in a general sense.
However, minimal-norm interpolation can still be uniformly stable
in a restricted sense, for instance, when the problem is
well-conditioned (on the subspace)~\citep{rangamani2023interpolating}.

In this work, we argue that the existence of a uniformly stable subnetwork
is sufficient for the stability of the entire neural network. However,
such a subnetwork might not be uniformly stable everywhere, but only
on some `nice' problems (such as separable classification problems).
A deeper understanding of this is left to determining whether and when
the subnetwork is stable.

\section{Omitted proofs}
\label{sec:omitted-proofs}

\subsection{Assumption \ref{existence low stable rank}}\label{proof:lemma0}

Low-rank bias (Assumption \ref{existence low stable rank}) can actually be proven under minimum-norm interpolation, but we note that, using solely the inequalities in \cite{timor2023implicit}, an exponentially deep network is needed to have guarantees. Empirically, low-rank layers and deep neural collapse already occurs with modest depths, suggesting that the inequalities in \cite{timor2023implicit} could be improved in some relevant data regimes. 

Invoking Theorem 4 in \cite{timor2023implicit}, we have that $\hat \th$ obtained by algorihtm $\cT$ for architecture $\mathbb{A} = \langle L, d, d_0 \rangle$ verifies:
\begin{align*}
    \frac{L}{\sum_{k=1}^L \left(\mathrm{S} (\hat \bW_k)) \right)^{-1}} \leq B^{\frac{L^*}{L}}
\end{align*}
thus
\begin{align*}
    \frac{1}{L} \sum_{k=1}^L \left(\mathrm{S} (\hat \bW_k)) \right)^{-1} \geq \frac{1}{B^{\frac{L^*}{L}}}
\end{align*}
Since the quantity on the left is an average, there must exist $1 \leq k \leq L$ such that 
\begin{align}
    \left(\mathrm{S} (\hat \bW_k)) \right)^{-1} \geq \frac{1}{B^{\frac{L^*}{L}}}
\end{align}
thus
\begin{align*}
    \mathrm{S} (\hat \bW_k)) \leq B^{\frac{L^*}{L}}~.
\end{align*}
By setting $L \geq L^* \frac{\log (B)}{\log (1 + a \cdot \epsilon)}$, we ensure that:
\begin{align*}
    \mathrm{S} (\hat \bW_k)) \leq 1 + a \cdot \epsilon~.
\end{align*}
We note $\hat \bW_k = \lambda_k \bu_k \bv_k^T + \hat \bW_k^\epsilon$ its decomposition following Definition \ref{def:sub-network}.
\subsection{Proof of Lemma \ref{low stable rank preserve signal}}\label{proof:lemma1}
Throughout the proof we drop dependence on $\th$, e.g.\ $N(\bx) \equiv N(\bx; \, \th)$.

  Introduce
  \begin{align*}
    b (\bx) := \frac{1}{\epsilon} \cdot \left(\net^{k+1:L} \left(\relu \left(\lambda_k \bu_k \bv_k\tp \net^{1:k-1}(\bx) \right) \right) - \net ( \bx ) \right)
  \end{align*}
   and we can thus write:
  \begin{align*}
      \net ( \bx ) + b (\bx) \cdot \epsilon = \net^{k+1:L} \left(\relu \left(\lambda_k \bu_k \bv_k\tp \net^{1:k-1}(\bx) \right) \right)~.
  \end{align*}
  \paragraph{Showing statement 1.}First consider the case when $\bv_k\tp \net^{1:k-1}(\bx) > 0$.
  Now, using the fact that $\relu$ is positively-homogeneous,
      \begin{align*}
       \relu \left(\lambda_k \bu_k \bv_k\tp \net^{1:k-1}(\bx) \right) = \bv_k\tp \net^{1:k-1}(\bx) \cdot \relu ( \lambda_k \bu_k)
      \end{align*}
      Therefore, using positive-homogeneity once again,
      \begin{align*}
        \net ( \bx )
        &= \net^{k+1:L}\pr{\bv_k\tp \net^{1:k-1}(\bx) \cdot \relu ( \lambda_k \bu_k)} - b (\bx) \cdot \epsilon \\
        &= \bv_k\tp \net^{1:k-1}(\bx) \cdot \underbrace{
          \net^{k+1:L} \left( \relu (\lambda_k \bu_k) \right)
          }_{C^+} - \, b (\bx) \cdot \epsilon
      \end{align*}
      where we note that $C^+$ can take a different sign since $C^+ = \bW_L\tp \net^{k+1:L-1} \left(  \relu (\lambda_k \bu_k) \right)$.

      %
      \paragraph{Showing statement 2.}
      Now, considering an alternative case $\bv_k\tp \net^{1:k-1}(\bx) \leq 0$, positive-homogeneity once again gives
      \begin{align*}
        \relu \left(\lambda_k \bu_k \bv_k\tp \net^{1:k-1}(\bx) \right) = - \bv_k\tp \net^{1:k-1}(\bx) \cdot \relu ( -\lambda_k \bu_k)
      \end{align*}
      and so
      \begin{align*}
        \net ( \bx )
        &=
          \net^{k+1:L}\pr{- \bv_k\tp \net^{1:k-1}(\bx) \cdot \relu ( - \lambda_k \bu_k)} - b (\bx) \cdot \epsilon\\
        &=
          \bv_k\tp \net^{1:k-1}(\bx) \cdot \underbrace{\pr{-\net^{k+1:L} \left( \relu (- \lambda_k \bu_k) \right)}}_{C^-} - b (\bx) \cdot \epsilon~.
      \end{align*}
      \paragraph{Showing statement 3 and 4.}
      It suffices to prove that for any input $\bx \in \mathbb{R}^{d_0}$,
  \begin{align*}
      \left| \net ( \bx ) -\net^{k+1:L} \left(\relu \left(\lambda_k \bu_k \bv_k\tp \net^{1:k-1}(\bx) \right) \right) \right| \leq a \, \epsilon \, B^L~.
  \end{align*}
  By $1$-Lipschitzness of $\relu$, we have:
  \begin{align*}
    &\| \relu \left(\bW_k \; \net^{1:k-1}(\bx) \right) - \relu \left(\lambda_k \bu_k \bv_k\tp \net^{1:k-1}(\bx) \right) \|\\
    &\quad\leq \| \bW_k^{\epsilon} \; \net^{1:k-1}(\bx) \|\\
    &\quad\leq \|\bW_k^{\epsilon}\|_F \, \| \net^{1:k-1}(\bx) \|\\
    &\quad\leq  a \, \epsilon \, \|\bW_k\|_F \, \|\net^{1:k-1}(\bx)\|~.
  \end{align*}
where the last inequality comes by the following observation:
\begin{align*}
  \mathrm{S}(\bW_k) \leq 1 + a \, \epsilon
  \quad \Longrightarrow \quad
  \|\bW_k^{\epsilon}\|_F \leq  a \, \epsilon \, \|\bW_k\|_F~.
\end{align*}
  At the same time, note that by $1$-Lipschitzness of $\relu$, and Cauchy-Schwartz inequality we have that for any $\bz, \bz' \in \R^{d_k}$,
  \begin{align*}
    |\net^{k+1:L}(\bz) - \net^{k+1:L}(\bz')| \leq \|\bz - \bz'\| \, \prod_{j=k+1}^L \|\bW_j\|_F~.
  \end{align*}
  Combining the above we have
  \begin{align*}
    &\left| \net ( \bx ) -\net^{k+1:L} \left(\relu \left(\lambda_k \bu_k \bv_k\tp \net^{1:k-1}(\bx) \right) \right) \right|\\
    &\quad=
      \left| \net^{k+1:L} \left( \relu \left(\bW_k \; \net^{1:k-1}(\bx) \right) \right) -\net^{k+1:L} \left(\relu \left(\lambda_k \bu_k \bv_k\tp \net^{1:k-1}(\bx) \right) \right) \right|\\
    &\quad\leq
      \pr{\prod_{j=k}^L \|\bW_j\|_F} \cdot
      a \cdot \epsilon \, \|\net^{1:k-1}(\bx)\|\\
    &\quad\leq
      \pr{\prod_{j=1}^L \|\bW_j\|_F} \cdot
      a \cdot \epsilon
  \end{align*}
  where the last inequality comes by Cauchy-Schwartz inequality, realizing that $\relu(|x|)=|x|$, and the fact that $\|\bx\| \leq 1$.
  \QED
\subsection{Proof of Lemma \ref{margin stable rank}}\label{proof:lemma2}
  First, notice that $\net (\cdot, \hat{\th})$ is a minimum-norm interpolant and so by Lemma \ref{lem:lemma14},
  the weight matrices $\widehat \bW_\ell$ have the same norm and thus verify:
  \begin{align*}
      \| \widehat \bW_\ell\|_F^2 = \frac{1}{L} \sum_{\ell = 1}^L \| \widehat \bW_\ell\|_F^2  \leq \frac{1}{L} \sum_{\ell = 1}^L \|\mathbf{W}_\ell^*\|_F^2 \leq \frac{1}{L} \cdot L \cdot B^2
  \end{align*}
  which gives us $\| \widehat \bW_\ell\|_F \leq B$.
  \paragraph{Proof of a lower bound.}
  Using the fact that $\net (\bx_i; \, \hat \th) = y_i$ and assumption that $|y_i| \geq 1$ for each $i$, we have:
\begin{align*}
  1 \leq |y_i| &= |\net (\bx_i; \, \hat \th) |\\
      &=
        \abs{
        \net^{k+1:L} \left( \relu \left(\hat \lambda_k \hat \bu_k \hat \bv_k\tp \, \net^{1:k-1}(\bx) + \widehat \bW_k^{\epsilon} \, \net^{1:k-1}(\bx) \right) \right)
        }\\
        &\stackrel{(a)}{\leq}
          \abs{
          \relu \left(\hat \lambda_k \hat \bu_k \hat \bv_k\tp \, \net^{1:k-1}(\bx) + \widehat \bW_k^{\epsilon} \; \net^{1:k-1}(\bx) \right)
          }
          \prod_{\ell = k+1}^{L} \| \widehat \bW_{\ell}\|_2\\
      &\stackrel{(b)}{\leq}
      \left(\left| {\hat \bv_k}\tp \net^{1:k-1} (\bx_i; \, \hat \th) \right| \| \hat \bu_k \| |\hat \lambda_k| + \left\|\widehat \bW_k^\epsilon \, \net^{1:k-1} (\bx_i; \, \hat \th) \right\| \right) \, \prod_{\ell = k+1}^{L} \| \widehat \bW_{\ell}\|_F
\end{align*}
where steps $(a),(b)$ comes by the Cauchy-Schwartz inequality and the fact that $\relu(|x|) = |x|$.

Thus,
    \begin{align}\label{inequality margin stable rank}
      \frac{1}{B^{L-k}} \leq \left| {\hat \bv_k}\tp \net^{1:k-1} (\bx_i; \, \hat \th) \right| \| \hat \bu_k\| | \hat \lambda_k | + \left \| \widehat \bW_k^\epsilon \, \net^{1:k-1} (\bx_i; \, \hat \th) \right\|
      \end{align}
      Furthermore by Lemma \ref{low stable rank preserve signal},
    \begin{align*}
      \left\|\widehat \bW_k^\epsilon \, \net^{1:k-1} (\bx_i; \, \hat \th) \right\|
      \leq
      a \, \epsilon \prod_{j=1}^k \|\bW_k^{\epsilon}\|_2
      \leq a \, \epsilon \, B^{k}~.
      \end{align*}
      By assumption of the Lemma we have $n \geq \left(2 \, M \, a \, B^L\right)^{\frac{1}{\alpha}}$, and therefore $\epsilon = \frac{M}{n^\alpha} \leq \frac{1}{2 \cdot a \cdot B^L}$,      
    which gives us that:
    \begin{equation*}
      \left\| \widehat \bW_k^\epsilon \, \net^{1:k-1} (\bx_i; \, \hat \th) \right\| \leq \frac{1}{2 \, B^{L-k}}~.
    \end{equation*}
    Plugged into \cref{inequality margin stable rank},  and using the fact that $ | \hat \lambda_k | \leq \|\widehat \bW_k\|_2 \leq  \| \widehat \bW_k\|_F \leq B$
and using that
    $\| \hat \bu_k \| =1$, we have:
    \begin{equation*}
      \left| {\hat \bv_k}\tp \net^{1:k-1} (\bx_i; \, \hat \th) \right| \geq \frac{1}{2 \, B^{L-k+1}}~.
    \end{equation*}
    \paragraph{Proof of an upper bound.}
    Similarly, using the Cauchy-Schwartz inequality,
    \begin{equation*}
      |\fsub_k(\bx; \, \hat \th)| = \left| {\hat \bv_k}\tp \net^{1:k-1} (\bx_i; \, \hat \th) \right| \leq \| \hat \bv_k \| \|\bx_i\|  \, \prod_{\ell = 1}^{k-1} \| \widehat \bW_{\ell}\|_F \leq B^{k-1}~.
    \end{equation*}
    \QED
\section{Experiments}\label{app:experiments}
\paragraph{Code} The code for our experiment is available at \url{https://anonymous.4open.science/r/stability_min_norm-E5B7}

\paragraph{Dataset}
We used binarized MNIST \citep{lecun1998gradient} and Fashion-MNIST \citep{xiao2017fashion} 
datasets where we assigned target value $-1$ for the first $5$ classes and $1$ for all other classes. The images were flattened to $1$-dimensional vectors.

\paragraph{Architecture}
For the experiment, we used a bias-free FCN with a width of $100$ and a depth of $8$. The model had a scalar output, in which MSE loss was used to classify between labels with output values $-1$ (first $5$ classes) and $1$ (latter $5$ classes). The weight matrix $W_k$ was initialized with Gaussian distribution with standard deviation $1/\sqrt{N_{k-1}}$.

\paragraph{Training} 
All models were trained with Adam \citep{KingmaB15} at a learning rate of $0.001$ until it reached over $99\%$ training accuracy. We used a weight decay of $0.005$ to obtain a minimum norm solution. At epochs $60$ and $120$, we divided the learning rate and weight decay by $5$.

\paragraph{Stability} To calculate the stability, we prepared $5$ models with identical architectures and initialization trained on $10,000$ mutually exclusive data points from MNIST. For all models, we perform SVD on the weight matrix of the $k^{th}$ layer ($W_k$) to obtain the largest right eigenvector $v_k$ and the subnetwork $\fsub_k$ (\cref{eq:sub-network}). Because the output of subnetworks can vary up to a scale, we normalize the subnetwork such that 
$\sum_{\bx \in \mathrm{TEST}} \fsub_k(\bx; \, \theta^{(i)})^2 = 1$.

The stability was measured by the absolute difference of the $\fsub_k$ on the test set.\footnote{The empirical $v_k$ may contain a global sign difference; we took the smallest stability obtained from $-v_k$ and $v_k$.}
\begin{equation*}
\mathrm{Stability} \approx \frac{1}{4}\sum_{i=2}^{5}\sum_{\bx \in \mathrm{TEST}} |\fsub_k(\bx; \, \theta^{(1)}) - \fsub_k(\bx; \, \theta^{(i)})|.
\end{equation*}
Likewise, the sign stability was assessed by the occurrences of identical labels.
\begin{equation*}
\mathrm{Sign~stability} := \frac{1}{4}\sum_{i=2}^{5}\sum_{\bx \in \mathrm{TEST}} |\mathrm{sign}(\fsub_k(\bx; \, \theta^{(1)}))- \mathrm{sign}(\fsub_k(\bx; \, \theta^{(i)}))|.
\end{equation*}

\paragraph{Computing resource}
Our experiments require a few minutes to $2$ hours of training on a GPU (RTX 3070 8GB) depending on the number of data points. All experiments require less than $3$ GB of memory. Experiments with a few data points were trained on CPU cluster which contains the following CPUs: Intel(R) Core(TM) i5-7500, i7-9700K, i7-8700; and Intel(R) Xeon(R) Silver 4214R, Gold 5220R, Silver 4310, Gold 6226R, E5-2650 v2, E5-2660 v3, E5-2640 v4, Gold 5120, Gold 6132.

\subsection{Experimental results}
\label{sec:add_plots}

In \Cref{fig:stability,fig:stability_by_n}, we present empirical evidence in support of our conjecture, and validating our main result: trained FCNs contain sub-networks (followed by a layer with low stable rank) whose stability is on par with that of the full network. Observe that the prediction stability of the
entire network is roughly of the same order (with respect to $n$) as
that of the sub-network stability: this is captured by similar slopes
in the log-scale, which corresponds to $- \alpha$ in
$\epsilon= n^{-\alpha}$ assumption.

Even though we empirically validate that our assumptions hold for
bias-free FCNs trained on several datasets, whether the condition holds
in larger architectures trained on more complex datasets requires more
empirical exploration. In the future, we aim to extend our results to
more complex scenarios and empirically explore the limit in which deep
neural networks satisfy our assumptions.

\begin{figure}[h]
    \centering
        \includegraphics[width=0.5 \textwidth]{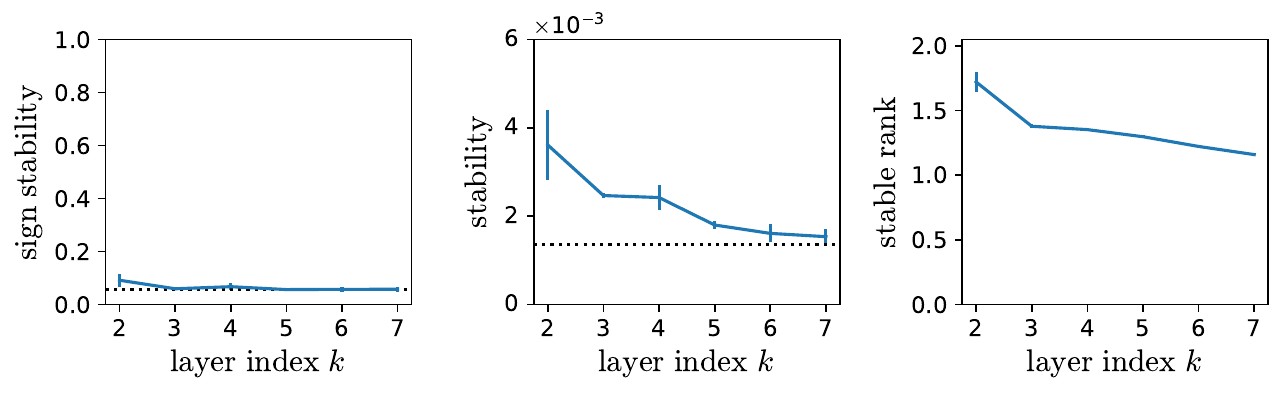}
    \caption{\textbf{Stability of sub-networks and stable rank of the layers.}
      We trained an $8$-layer FCN on a uniformly drawn $10^4$ MNIST sample by minimizing a mean square error (MSE) loss to near zero, classifying the first $5$ classes as $-1$ and others as $1$.
      We performed multiple trials, where
      each trial is with identical initialization and a different portion of the training set is replaced for each trial.
      The error bars are $1$ standard deviation of the trial.
      Using the models, we measured the sign stability \textbf{(left)}, i.e. $|\mathrm{sign}(\fsub_k(\bx; \, \hat \th)) -\mathrm{sign}(\fsub_k(\bx; \, \hat \th^{(i)}))|$, stability \textbf{(middle)}, and the stable rank of weight matrix \textbf{(right)} for each sub-network $\fsub_k$ for $ 2 \leq k \leq 7$. The horizontal dotted lines are the (sign) stability of the full network.  For the details of the experiment, link to our code, and additional experiments on
      Fashion-MNIST, see \cref{app:experiments}.
    }\label{fig:stability}
\end{figure}
%
\begin{figure}[h]
    \centering
        \includegraphics[width=0.5 \textwidth]{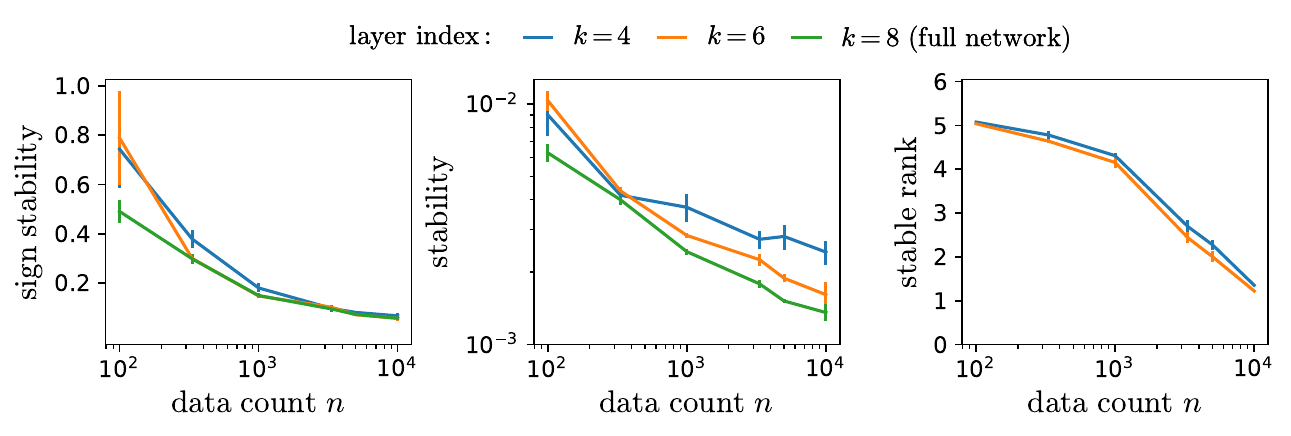}
        \caption{\textbf{Stability as a function of number of data points} We followed the same setting as in \cref{fig:stability} while varying training data set sizes. Both the sign stability \textbf{(left)} and stability \textbf{(middle)} of sub-networks (blue and orange) decay at a rate similar to that of the full network (green). The stable rank of weight matrices \textbf{(right)} also decreases as a function of $n$, suggesting that \Cref{existence low stable rank} holds in the large $n$ limit.
          Observe that the slopes for the sub-networks and the full network are similar which validates that the respective stabilities have the same dependency in $n$ (\cref{stable sub network with stable rank}).} \label{fig:stability_by_n}
\end{figure}

In \cref{fig:fmnist_stability,fig:fmnist_stability_by_n}, we repeat the experiments for \cref{fig:stability,fig:stability_by_n} for Fashion-MNIST and obtain equivalent results. In \cref{fig:compression}, we plot the performance of the sub-networks in \cref{fig:fmnist_stability}.

\begin{figure}
    \centering
        \includegraphics[width=0.9 \textwidth]{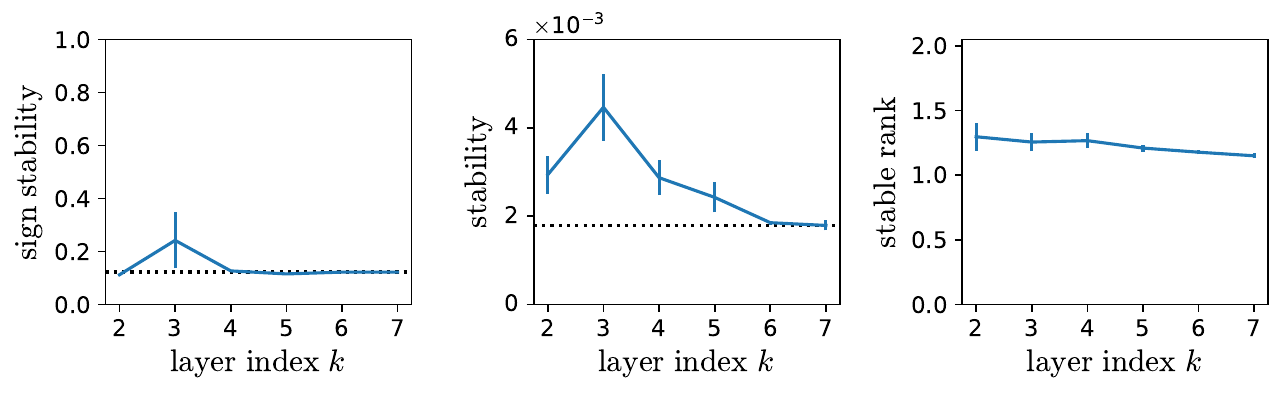}
    \caption{\textbf{Stability of subnetworks and stable rank of the layers (Fashion-MNIST).} We repeat the experiment of \cref{fig:stability}, but on Fashion-MNIST dataset. In agreement with the experiment for MNIST, FCN contains stable subnetwork and low-rank layers. }\label{fig:fmnist_stability}
\end{figure}

\begin{figure}
    \centering
        \includegraphics[width=0.9 \textwidth]{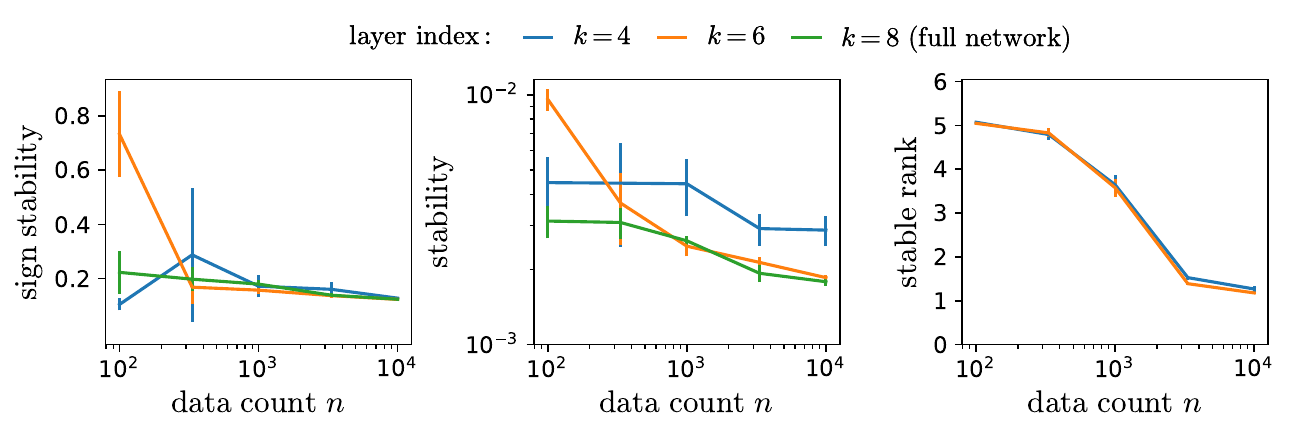}
    \caption{\textbf{Stability as a function of data points (Fashion-MNIST)} We repeat the experiment of \cref{fig:stability_by_n}, but on Fashion-MNIST dataset. As $n$ increases, the stability of the subnetwork $\fsub_6$ (orange) is similar to the stability of the whole network (green) and the stable rank of $W_6$ also converges to $1$.}\label{fig:fmnist_stability_by_n}
\end{figure}

\begin{figure}[htp]
    \centering
        \includegraphics[width=1.0 \textwidth]{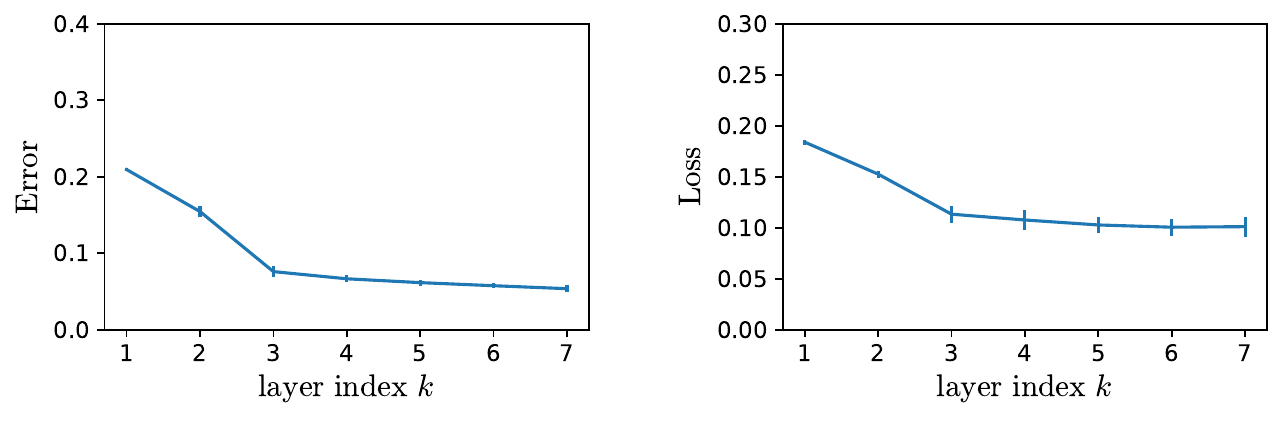}
    \caption{\textbf{Layer compression.} Using the models trained for \cref{fig:stability}, we send the output of the $k^{\mathrm{th}}$ layer to an auxiliary linear layer, which is trained using Adam over $200$ epochs. We measure the test error \textbf{(left)} and the test loss \textbf{(right)}, which resembles the stability in \cref{fig:stability}. }\label{fig:compression}
\end{figure}

\end{document}